\documentclass[lettersize,journal]{IEEEtran}
\usepackage{amsmath,amsfonts}
\usepackage{algpseudocode}
\usepackage{algorithm}
\usepackage{hyperref}
\usepackage{array}
\usepackage[caption=false,font=normalsize,labelfont=sf,textfont=sf]{subfig}
\usepackage{textcomp}
\usepackage{cite}
\usepackage{amsthm}
\newtheorem{theorem}{Theorem}
\newtheorem{lemma}[theorem]{Lemma}
\usepackage{stfloats}
\usepackage{url}
\usepackage{verbatim}
\usepackage{graphicx}
\usepackage{cite}
\usepackage{amssymb}
\begin{document}

\title{Online Obstacle evasion with Space-Filling Curves}

\author{Ashay Wakode$^{1}$, Arpita Sinha$^{2}$
\thanks{$^{1}$ Ashay Wakode is with Mechanical Engineering Department, Indian Institute Of Technology Bombay, Mumbai, India
{\tt\small  {ashaywakode}@iitb.ac.in  }}

\thanks{$^{2}$ Arpita Sinha is with System and Controls Department; Indian Institute Of Technology Bombay, Mumbai, India  \\
{\tt\small  {arpita.sinha}@iitb.ac.in  }
}

}


\maketitle
\begin{abstract}
The paper presents a strategy for robotic exploration problems using Space-Filling curves (SFC). The region of interest is first tessellated, and the tiles/cells are connected using some SFC. A robot follows the SFC to explore the entire area. However, there could be obstacles that block the systematic movement of the robot. We overcome this problem by providing an evading technique that avoids the blocked tiles while ensuring all the free ones are visited at least once. The proposed strategy is online, implying that prior knowledge of the obstacles is not mandatory. It works for all SFCs, but for the sake of demonstration, we use Hilbert curve. We present the completeness of the algorithm and discuss its desirable properties with examples. We also address the non-uniform coverage problem using our strategy. 
\end{abstract}
\section{Introduction}
\IEEEPARstart{I}{n} 1878, George Cantor demonstrated that an interval $I = [0,1]$ can be mapped bijectively onto $[0,1] \times[0,1]$. Later, G. Peano discovered one such mapping that is also continuous and surjective; the image of such mapping when parameterized in the interval $I$ to higher dimensions ($\mathbb{R}^{n}$) is known as Space-Filling Curve (SFC). More SFCs were later discovered by E. Moore, H. Lebesgue, W. Sierpinski, and G. Polya \cite{sagan,bader}. See Fig \ref{fig:1}.
SFCs have some interesting properties - 
\begin{itemize}
    \item \textbf{Self-Similar} - Each curve is made out of similar sub-curves
    \item \textbf{Surjective Map} - SFC passes through every point of $\mathbb{R}^{n}$
    \item \textbf{Locality Preserving} - Two points close by in $I$ map to
close by points in $\mathbb{R}^{n}$
\end{itemize}
\par
Due to the above properties, SFCs have been used in many applications - data collection from sensor network \cite{sensor1, sensor2}; ordering meshes of complex geometries \cite{bader} and many more. An approximate solution to Travelling Salesman Problem (TSP) can be found using Hilbert's Space Filling curve \cite{tsp}. Space Filling Tree analogous to SFCs having tree-like structure have been proposed for sampling-based path planning \cite{SFCtree}, as opposed to traditional methods like Rapid-exploring Random Trees (RRTs) \cite{RRT}.
\par
\begin{figure}[t]
    \centering
    \includegraphics[scale=0.3]{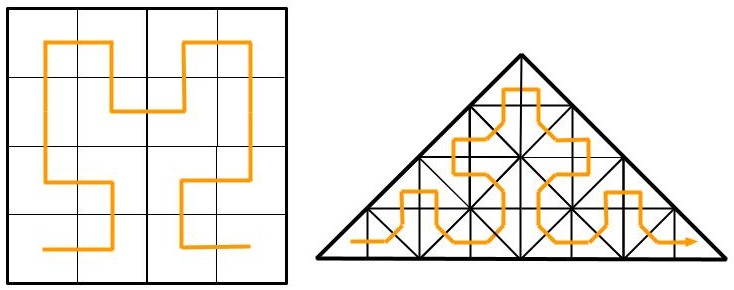}
    \caption{Hilbert curve and Sierpinski curve}
    \label{fig:1}
\end{figure}
One of the major applications of SFCs is robotic exploration. In robotic exploration problem, a single or group of robotic agents are deployed to search, survey or gather information about a specific region while avoiding obstacles. Robotics exploration is one of the sub-problems of the larger Coverage Planning Problem (CPP), wherein the agent is bestowed with the task of visiting all points in a given region or volume while avoiding obstacles \cite{galceran, choset, uavsurvey}. Numerous approaches for CPP already exists - Graph based, Grid based, Neural-Network based, Cellular decomposition based \cite{galceran}. Each of these approaches can be used for robotic exploration problem. 
\par
SFCs have been used for robotics exploration and the reasons making it eligible are: 
\begin{itemize}
\item \textbf{Complete and Robust} : Coverage using SFCs is time complete and robust to failure \cite{spires}
    \item \textbf{Extensible} : 
    \begin{itemize}
        \item Dimension: The strategy developed for 2D can extended to 3D since similar grammar exists for the construction of SFC in both dimensions \cite{bader}
        \item Number of search agents : Coverage using SFCs have been shown to be complete and robust when multiple agents are used \cite{spires}
        \item Irregular Area : The generalized version of SFCs aka Generalized SFC (GSFC) can span irregular quadrilateral and triangles as opposed to SFCs which map regular shapes like square or isosceles right triangle \cite{bader}
    \end{itemize}

\item \textbf{Non-Uniform coverage} : SFCs can be easily used in non-uniform coverage scenarios requiring some parts to be searched more rigorously than others. On top of that Hilbert curve has been shown to be more efficient in time/energy than the lawnmower's path \cite{sadat2}
 \end{itemize}
 \par
A 2D area with multiple obstacles of arbitrary shape and size is explored using a SFC. The area is tessellated 
(referred to as waypoints) and tiles/cells are connected in the order determined by SFC. The tessellation with 
obstacle is considered unreachable. A online strategy is suggested for a search agent to explore the area using SFC and evade the obstacle on the go. It suggests modified waypoint when unreachable waypoints are encountered. It guarantees that all the waypoints reachable from the initial waypoint are explored. The strategy works for all SFCs, but for the sake of demonstration Hilbert's curve is used in this paper.
\par
The rest of this paper is organized into six sections: Section II elaborates on the related work and the need for a new strategy. Section III discusses preliminaries on SFCs. Section IV details the strategy, investigates its completeness and discusses its useful properties. Section V compiles the simulation result. Section VI concludes the paper and talks about the limitations and possible future work. 

\section{Related Work}
There is sizable literature dealing with using SFC in the robotic exploration problem.  \cite{spires} suggests the use of a swarm of mobile robots for exploration, each mobile robot following a SFC. This approach is efficient in terms of energy, robust to failures and assures coverage in finite time, but considers an obstacle-free environment. \cite{sadat1} proposes a new method for non-uniform coverage planning for UAVs. \cite{sadat2} takes the work further and uses the Hilbert curve for non-uniform coverage which is optimal as opposed to existing methods. \cite{hawk} introduced a UAV system to conduct aerial localization that uses Moore's SFC. However, \cite{sadat1, sadat2, hawk} does not consider obstacles. 
\par
\cite{tiwari} introduced path planning approach for SFCs with obstacles and proved the optimality for specific obstacle configurations, due to the existence of a Hamiltonian path for such obstacle configurations. \cite{ban} introduced an algorithm to construct SFCs for sensor networks with holes. The algorithm can be used for motion planning with obstacles while using SFC. However, the solutions proposed in \cite{tiwari, ban} require the knowledge of obstacles before starting the exploration.
\par 
\cite{nair} formulated an online obstacle evasion strategy for Hilbert curve with only one waypoint blocked by an obstacle. \cite{joshi} carries forward the work and suggests a strategy capable of evading two neighboring waypoints on a Hilbert curve. \cite{wakode} talks about online obstacle evasion for aerial vehicle covering a region using the Sierpinski curve, but the obstacles need to block disjoint cells.
\par
Early work on the topic did not consider obstacles in the environment. Later works considered obstacles but required knowledge about the obstacles or were restricted to a particular class of obstacles and SFC used. Therefore an online strategy capable of evading arbitrary obstacles while using any SFC is the required next step.

\section{Preliminaries on SFCs}
SFCs are defined as, \\
\textbf{Definition} : Given a mapping $f : I \rightarrow Q \subset \mathbb{R}^{n}$, with $f_{*}\textit(I)$ as image. $f_{*}\textit(I)$ is called a SFC, if $f_{*}\textit(I)$ has Jordan content (area, volume, ..) larger than 0 \cite{bader}. 
\par 
In this paper, approximations of SFCs are used. Approximate SFCs are constructed by dividing $\mathbb{R}^{n}$ and connecting the centers of the cells by a continuous piecewise straight line. The way the divisions are connected depends on the grammar of the SFC being used.
\par
\begin{figure}[t]
    \centering
    \includegraphics[width=3in]{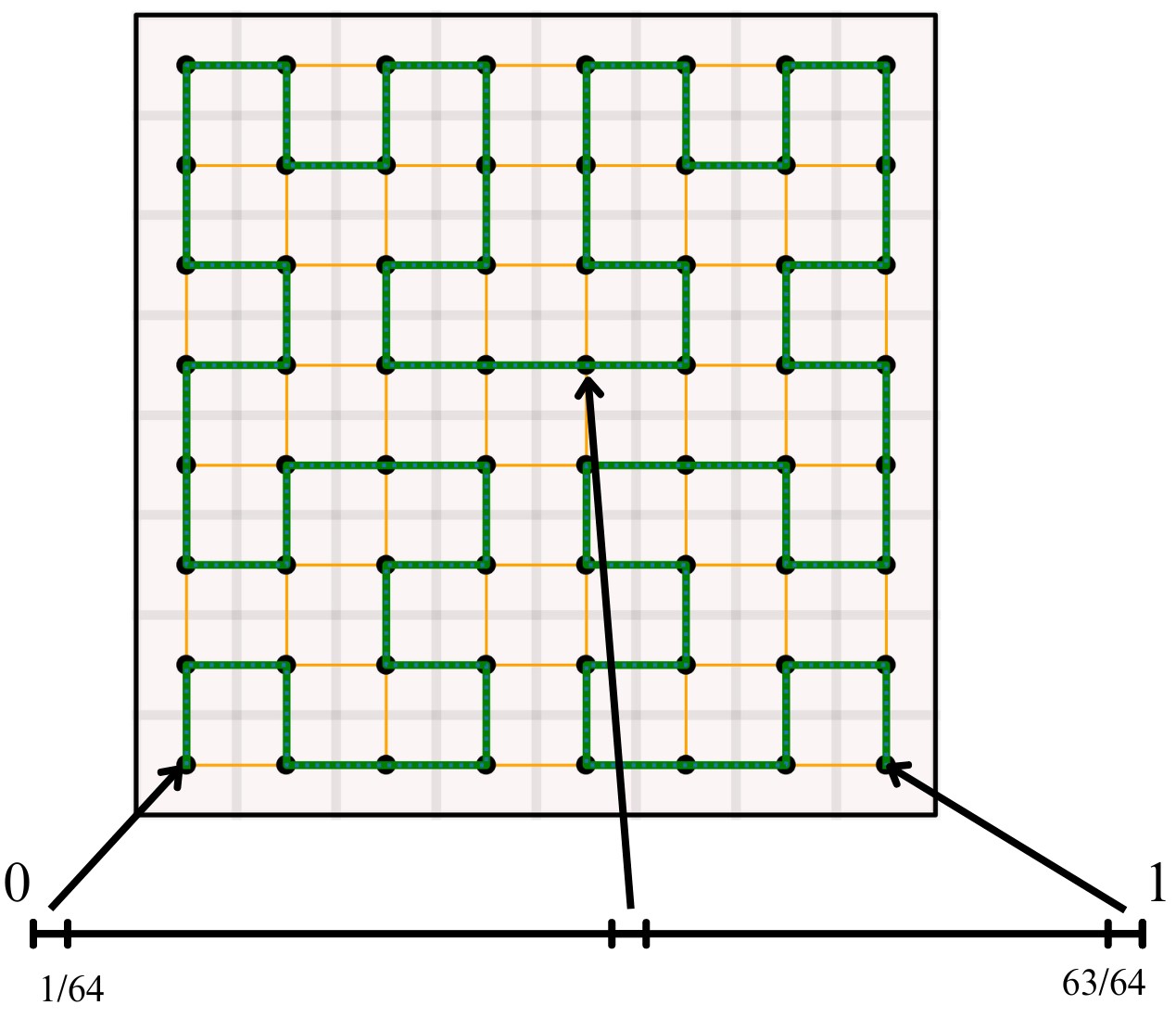}
    \caption{Hilbert curve as mapped from $I$}
    \label{fig:mapping}
\end{figure}
\begin{figure*}[t]
    \centering
    \includegraphics[width=4.8in]{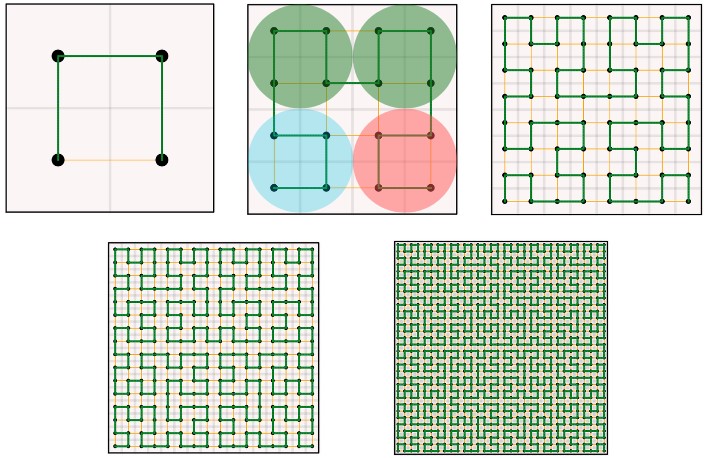}
    \caption{Hilbert curve iterations 1 to 5; Colored circles represent specific translation + rotation rules for creating $n^{th}$ iteration from $n-1^{th}$ iteration}
    \label{fig:iterations}
\end{figure*}
\par
Here in Fig \ref{fig:mapping}, it is seen that a square is divided N times and the centers are connected to generate Hilbert curve. In this case, hilbert curve is mapped onto a square in $\mathbb R^{2}$. The degree to which the square is divided is quantified by the term - "iteration". As iteration goes to infinity, approximate SFC visits every point in the square; this curve is termed as SFC. Thus explains the \textbf{surjective mapping} property of SFC. A Rigorous mathematical explanation can be found in \cite{sagan,bader}. The SFCs in Fig \ref{fig:1}, \ref{fig:mapping} are approximate but are referred to simply as SFC in literature. 
\par
Moreover, approximate SFCs are sufficient for the exploration given that the robotic agent has a search radius enough to cover the entire division (area or volume) while at its center.
\par
Furthermore, it can be seen in Fig \ref{fig:iterations} that iteration $N$ is constructed by rotation and translation of iteration $N-1$. Hence explaining the \textbf{self-similar} property. The grammatical way of constructing SFCs uses this fact.
\par SFCs are H\"{o}lder continuous mappings. Given, 
\begin{equation*}
    x, y \in I
\end{equation*}
and $f$ is SFC with $f : I \rightarrow Q \subset \mathbb{R}^{n}$, if,
\begin{equation*}
     \lVert f(x)- f(y) \rVert _{2} \leq C |x-y|^{r} 
\end{equation*}

$f$ is said to be H\"{o}lder continuous with exponent $r$, if a constant $C$ exists for all $x,y$. The RHS can be interpreted as the distance between $x$ and $y$, while LHS is the distance between points in SFC. Here, the distance between points on SFC is bounded by the distance between the points they are mapped from, explaining the \textbf{Locality Preservation}.
\par
Lastly, any problem in Q $\subset \mathbb{R}^{n}$ mapped by SFC can be transformed into a problem in $I$. And often solving the problem in $I$ is straightforward than solving it in the original space $Q$. This technique is referred to as \textbf{General Space-Filling Heuristics} (\textbf{GSFH}). 
\section{Problem Formulation and Proposed Solution}
A robotic agent with sensing radius $s$, is given a task to explore a 2D region using SFC and avoiding the obstacles when detected. The region to be explored can have any shape and size. But, SFCs only span regular shapes like squares (Hilbert curve) and isosceles right triangle (Sierpinski curve). Here, it is assumed that the shape to be explored is regular, since an irregular shape can be decomposed into regular shapes and explored separately. Extension of SFCs to General SFCs \cite{bader}, which map arbitrary quadrilaterals and triangles, also exists and can be used. For this paper, we look into SFCs only. 
\par
The strategy suggested can be used for any SFC. For demonstration the results are worked out for Hilbert curve. A square with area $A$, has obstacles at unknown locations. Without loss of generality, let the robotic agent start at left-hand bottom corner of the square. It follows Hilbert curve to explore the region while avoiding the static obstacles and get to the right-hand bottom corner. The obstacle can have arbitrary shape and size. The iteration ($k \in N$) of Hilbert curve used is such that the agent can scan an entire cell while at its center. Mathematically, 
\begin{equation}
 \label{eq1}
    k \geq \lceil log_{2}(A / s \sqrt{2} - 1) \rceil
\end{equation}
$Note$ : This relation has been found out by comparing sensing radius and diagonal of cell for Hilbert curve. Similar relation for other SFCs can be found out trivially.
\par
Here, the robotic agent detects an obstacle just before going to the cell with an obstacle. It is also assumed that the agent has the capability to memorize the locations of detected obstacles.
\subsection{Proposed Solution}
Given an SFC mapping an area, the center of each cell is identified as waypoint and are numbered from $0$ to $N-1$ based on SFC. The goal of the robotic agent is to start at $0$ and visit all the unblocked reachable waypoints, This means going to a waypoint with the maximum number possible. Certain region may not be blocked by the obstacle but can still be unreachable when it is blocked by obstacles on all sides, like in the case of annular shaped obstacle. 
\par 
Here, a graph theoritic solution is presented. Nomenclature used :
\begin{itemize}
    \item $G = $ Graph with waypoints of SFC as vertices; Adjacent vertices are connected by the edge; No waypoint is assumed to be blocked by an obstacle in $G$. $G$ is a dual graph of SFC decomposition, similar to one used in \cite{tiwari}.
    \item $O = $ Set of vertices where obstacles have been detected  
    \item $V = $ Set of visited vertices 
    \item $A(V, G) = $ Set of vertices in $G$ adjacent to vertices in $V$ but not present in $V$
\end{itemize}
The strategy is presented as a pseudocode in Algorithm \ref{alg:main}. The agent uses the strategy at each waypoint to decide the next waypoint. While at waypoint $c$, the agent knows $G$ given the SFC, $V$ from the visited nodes and $O$ from the detected obstacles till that time, and needs to decide about next waypoint to visit.
If all the waypoints adjacent to $V$ in $G$ are blocked, no waypoint remains to be visited and the search can be terminated (line $4$). Suppose there exists a set of waypoints that can be visited. In that case, the strategy suggests minimum numbered waypoint $p$, and the shortest path $R$ to $p$ is found using Dijkstra's algorithm (any shortest path finding algorithm can be used).  $R$ refers to an array of waypoints starting with $c$ and terminating at $p$. The agent checks if $p$ is blocked while at the pen-ultimate waypoint (line $7$). If $p$ is found to be blocked, it is added to $O$ and the strategy starts again at step $4$ (line $8$ to $10$). Otherwise, $p$ is reached and added to $V$ (line $12$ and $13$). The algorithm is used again to decide about the next waypoint. The strategy is explained using an example in the next section.

\begin{algorithm}
\caption{Online Strategy for Obstacle evasion}\label{alg:main}
\begin{algorithmic}[1]
\State Input : $G$, $V$, $O$
\State Output : Next waypoint
\State Initialization : Current waypoint ($c$)
\If{$A(V, G) - O \neq \emptyset$}
    \State $p = min(A(V, G) - O)$ \Comment{min numbered element}
    \State $R =$ shortest route to $p$ starting from $c$
    \While{Agent at $R[-2]$} \Comment{second last element of $R$}
        \If{$p$ is blocked $\lor$  $p \in O$} 
        \State $O \gets O \cup \{p\}$  
        \State Go back to step $4$
        \Else 
        \State Next waypoint $= p$ 
        \State  $V \gets V \cup \{p\}$
        \EndIf
        \EndWhile
\Else 
\State All the reachable waypoints visited
        \EndIf
\end{algorithmic}
\end{algorithm}

\begin{lemma}
An agent starting at waypoint $H$ and following the proposed strategy will visit all the waypoints connected to $H$.
\end{lemma}
\begin{proof}
Let us prove the lemma by contradiction. Assume here a waypoint $J$ (not blocked by an obstacle) connected to $H$, but the agent terminated the exploration without visiting $J$. This can happen only if,
\begin{equation*}
    A(V, G) - O = \emptyset
\end{equation*}
But, 
\begin{equation*}
    A(V, G) - O = \{J\}
\end{equation*}
Since, $J$ is connected to $H$ and not visited. Therefore, a contradiction.
\end{proof}
Here, the agent will explore the area spanned by waypoints given the sensing radius enough to cover the entire cell (Eq. \ref{eq1}). But, this does not guarantee complete search. 
\par
\begin{figure*}[t]
\centering
\includegraphics[width=4.5in]{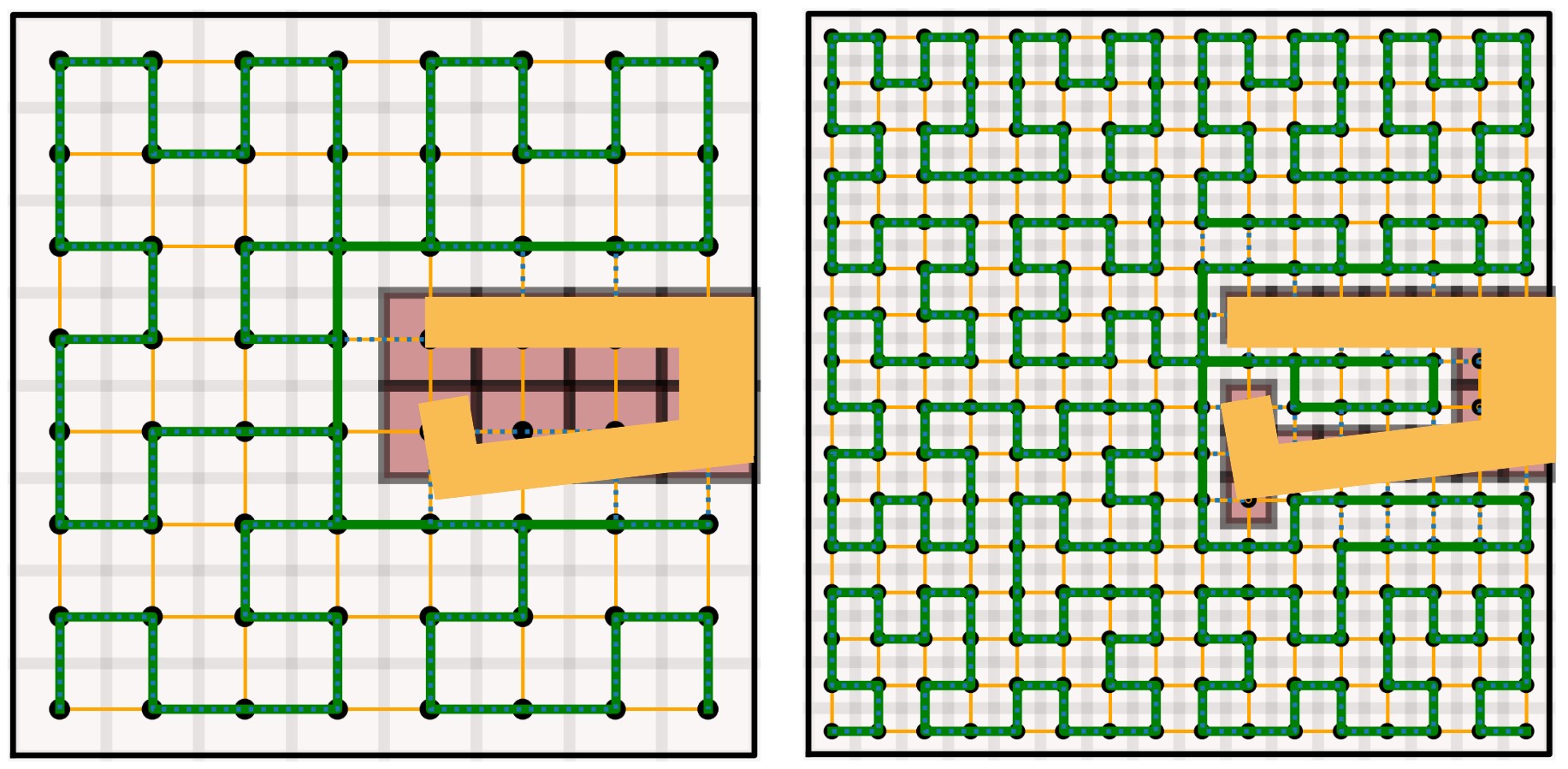}
\caption{Region blocked by tight space reachable through the use of higher iteration (4) of Hilbert curve}
\label{inv}
\end{figure*}
In a situation with tight spaces, certain reachable regions may get blocked due to the iteration of SFC used. In such cases, the iteration of SFC can be increased. This would yield a better approximation of the obstacle and hence better coverage. This is illustrated in Fig \ref{inv}. \\
$Remark$ : There can be multiple strategies to select the next waypoint among the connected waypoints. The proposed strategy ensures that SFC is followed in the regions without obstacles. 
\subsection{Discussion}
\begin{itemize}
    \item The strategy works well with both normal  and sparse obstacles, which is not the case with Lawnmower's path which performs poorly with sparse obstacles.
    \item The path modification occurs locally. Given a obstacle, the strategy suggests path changes closer to it. Meaning SFC retains its original form in the rest of the region.
    \item Higher iterations of SFCs are preferred for sparse obstacles while lower iterations are preferred for obstacles with larger size. Higher iterations offer a more agile path and less occlusion than lower iterations when dealing with sparse obstacles. At the same time, they offer a longer path and may not be desired for larger obstacles.
    \item The path generated using the strategy can be adapted for different parts of the region to be explored, given the knowledge of the nature of obstacle (Sparse or Normal). Further, the agent may want to explore certain parts with more rigor. In such cases, higher iteration may give the agent more time and focus the search on smaller cells/tiles \cite{sadat1, sadat2}. This can be easily achieved with proposed strategy. An Example of such a situation is introduced in the next section\\
\end{itemize}
Obstacles comparable to the sensing radius of the agent are considered normal. While the obstacles much smaller than the sensing radius are considered sparse. Here it is assumed that the agent can modify the sensing radius, as is the case in \cite{sadat1, sadat2} with regards to aerial vehicles. Otherwise, the agent can discard the data collected outside the required sensing radius.  
\section{Results}
The presented strategy was implemented in \textbf{Python} (version $3.8$). \textbf{iGraph} \cite{igraph} was used for graph operations (version $22.0.4$). \textbf{Hilbert curve} library \cite{hilbert} was used for plotting the Hilbert curve (version $2.0.5$). The code for implementation and generated figures are available at - \url{https://github.com/wakodeashay/SKC}. Let's look at examples,
\begin{itemize}
    \item \textbf{Normal Obstacles} : The strategy was implemented for a situation with normal obstacles. Figure \ref{normal} shows two normal sized obstacles blocking the SFC. The modified path is also shown.  
    \par
    Here, while at $0$,  
    \begin{equation*}
    A(V, G) - O = \{1, 2, 3\}
\end{equation*}
the agent searches for obstacle and goes to $1$. Next,
\begin{equation*}
    V = \{0, 1\},   A(V, G) - O = \{2, 3, 13\}
\end{equation*}
agent searches for obstacle and goes to $2$. When no obstacle is found, the agent goes to the next waypoint in the SFC due to the minimum rule. This happens till $21$, when obstacle is detected at $22$. At this point,
\begin{equation*}
    V = \{0, 1, ..., 21\}
\end{equation*}
\begin{equation*}
   A(V, G) - O = \{23, 29, 30, 31, 32, 53, 54, 57, 58\}
\end{equation*}
So, the agent tries to reach $23$ but encounters obstacle while at $20$. Next, it goes to $29$ from $20$ using shortest path. Now,
\begin{equation*}
    V = \{0, 1, ..., 21, 29\}
\end{equation*}
\begin{equation*}
   A(V, G) - O = \{24, 28, 30, 31, 32, 53, 54, 57, 58\}
\end{equation*}
The agent decides to go $24$ but detects an obstacle while at $29$. After which the agent goes to $28$,
\begin{equation*}
    V = \{0, 1, ..., 21, 29, 28\}
\end{equation*}
\begin{equation*}
   A(V, G) - O = \{27, 30, 31, 32, 53, 54, 57, 58\}
\end{equation*}
Hence, the agent goes to $27$ and $26$ is added to $A(V, G) - O$, which is reached finally. Now,
\begin{equation*}
    V = \{0, 1, ..., 21, 29, 28, 27, 26\}
\end{equation*}
\begin{equation*}
   A(V, G) - O = \{25, 30, 31, 32, 53, 54, 57, 58\}
\end{equation*}
The agent decides to go $25$ and detects obstacle while at $26$. Finally, the agent goes to $30$, $31$ and so on following the SFC until next obstacle is detected.

\begin{figure*}[t]
\centering
\includegraphics[width=2.6in]{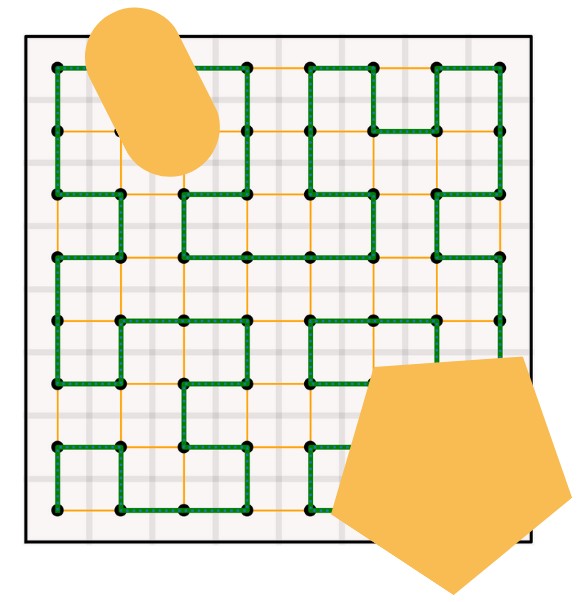}  \includegraphics[width=2.4in]{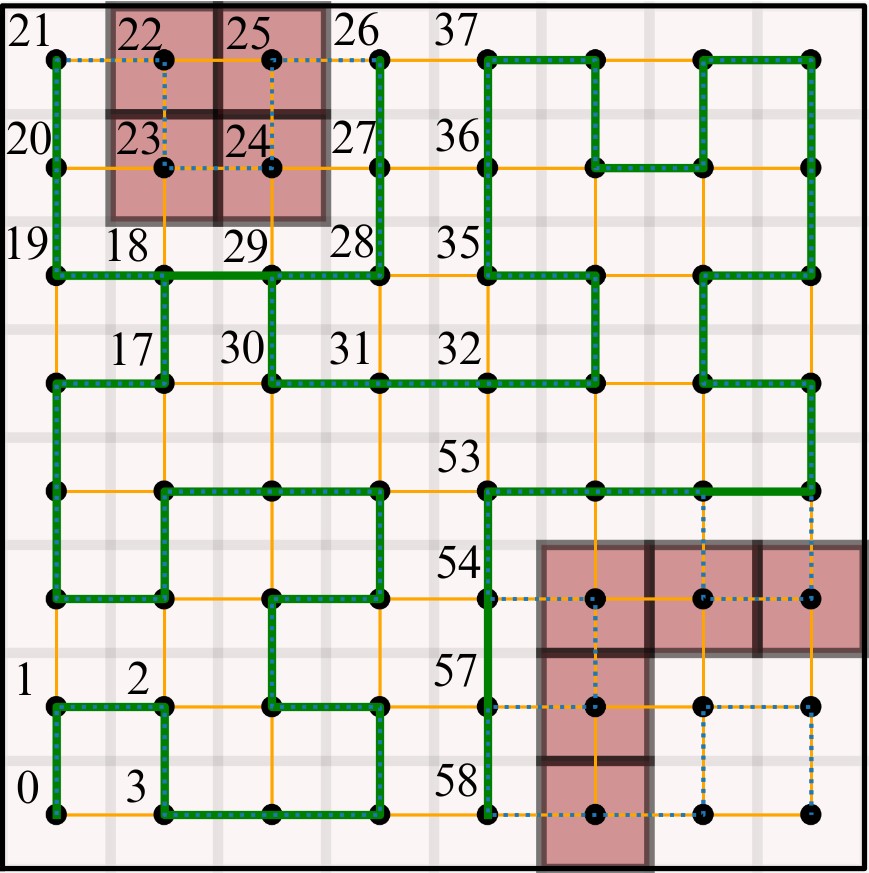} 
\caption{Hilbert curve with normal sized obstacles; Modified path, Brown represent waypoints with detected obstacle}
\label{normal}
\end{figure*}
    \item \textbf{Sparse Obstacles} : Figure \ref{spa} shows three scenarios with obstacle placed sparsely in the area to be explored.
\begin{figure*}[t]
\centering
\includegraphics[width=6in]{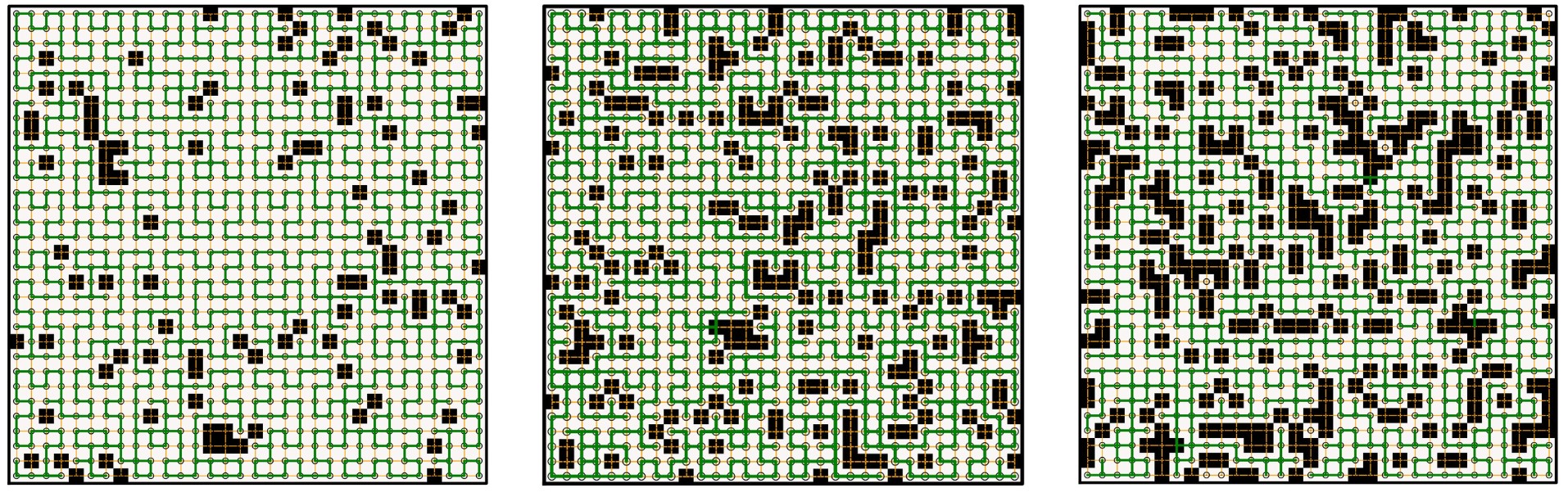}
\caption{Obstacle evasion scenarios with sparse obstacles; 9.7\% (100), 19.5\% (200) and 29.2\% (300) of total waypoints (1024) blocked respectively}
\label{spa}
\end{figure*}
    \item \textbf{Non-Uniform Coverage} : In certain cases user may want to search some parts more rigorously. This can be done using higher iteration SFC, since the agent spends more time with higher iteration and the search is focussed on smaller radii. Nonetheless, if the user knows about the nature of obstacle (sparse or normal), the iteration number can be increased or reduced on the go. This is illustrated in \ref{non-uni}.
    \par
    $Remark$ : In certain situations, agent is unable to get to the last point of the SFC. This is seen in the right-hand lower quadrant of fig \ref{non-uni}. The agent ends his journey at point $A$. Therefore, it does not know if it can directly move to the first waypoint of the next SFC (could be blocked). In this case, agent can be moved to a waypoint ($B$) on the shared edge closest to $A$. Eventually, moving to the closest waypoint on the next SFC and start using the proposed strategy. The agent will move to the first waypoint if it is connected to $C$, in which case it, can start the exploration afresh if required by the user or else continue using the strategy.

\begin{figure*}
\centering
\includegraphics[width=5.4in]{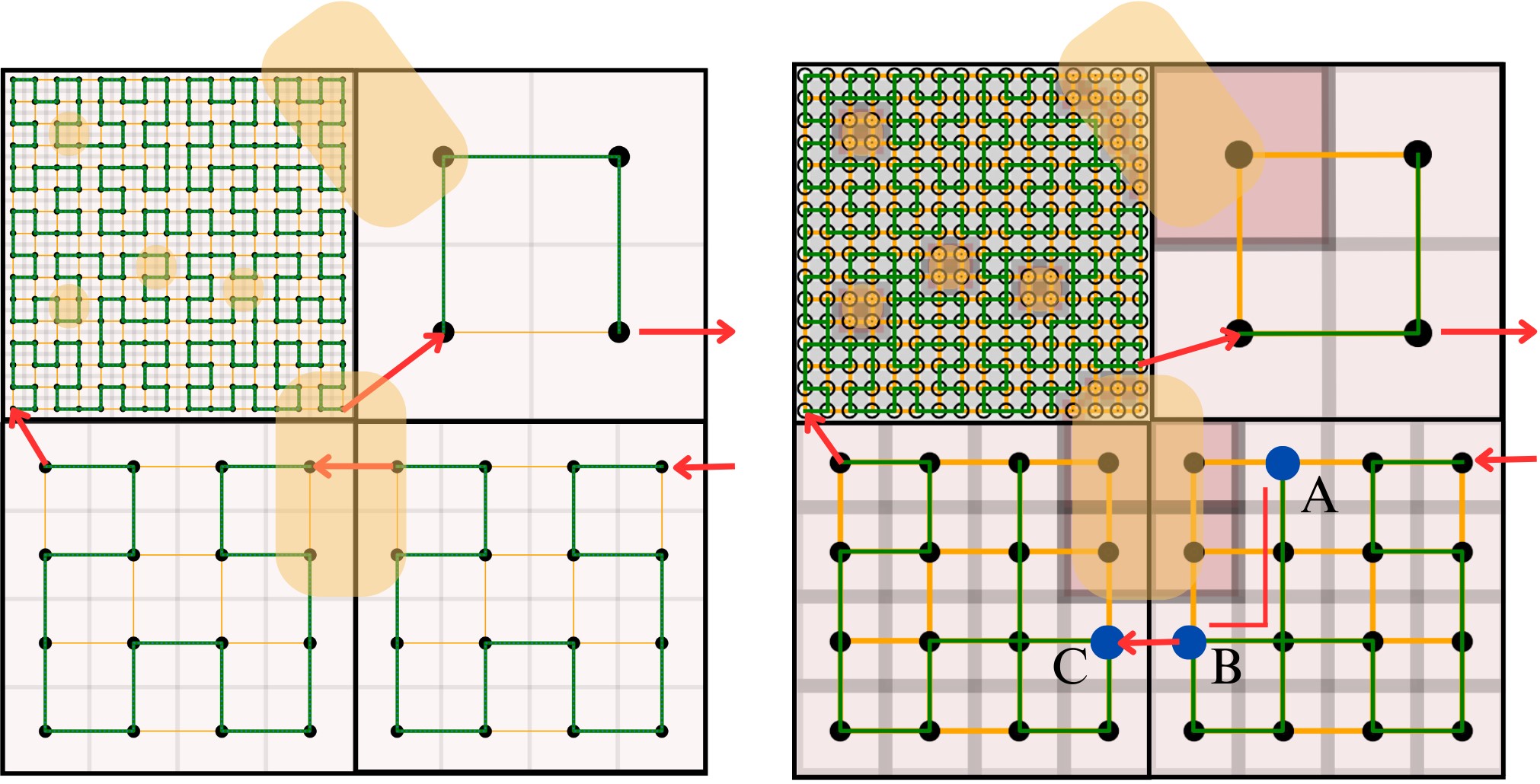}
\caption{Inward red arrow shows starting position, Outward red arrow shows terminal position}
\label{non-uni}
\end{figure*}
\end{itemize}

\section{Conclusion}
This paper discusses a strategy for online obstacle evasion on Space-Filling curve. Preliminaries on Space-Filling curves were introduced, after which the problem was formally put down. The strategy is presented and the completeness is also touched upon. Later, it was validated on examples. The presented strategy is a high-level waypoint planner. The path needs to be smoothened for agents with non-holonomic constraints. Finally, The strategy can be optimized to reduce the total number of visits to waypoints. Extension to 3D space, environment with dynamic obstacles and Multi-agent scenarios could be interesting future directions.


\begin{thebibliography}{1}
\bibliographystyle{IEEEtran}

\bibitem{sagan}
Sagan, H., 2012. Space-filling curves. Springer Science \& Business Media.

\bibitem{bader}
Bader, M., 2012. Space-filling curves: an introduction with applications in scientific computing (Vol. 9). Springer Science \& Business Media.

\bibitem{sensor1}
Yan, Y. and Mostofi, Y., 2014, December. An efficient clustering and path planning strategy for data collection in sensor networks based on space-filling curves. In 53rd IEEE Conference on Decision and Control (pp. 6895-6901). IEEE.

\bibitem{sensor2}
Yan, Y. and Mostofi, Y., 2016. Efficient clustering and path planning strategies for robotic data collection using space-filling curves. IEEE Transactions on Control of Network Systems, 4(4), pp.838-849.

 \bibitem{tsp}
 Bartholdi III, J.J. and Platzman, L.K., 1982. An O (N log N) planar travelling salesman heuristic based on spacefilling curves. Operations Research Letters, 1(4), pp.121-125.

 \bibitem{SFCtree}
 Kuffner, J.J. and LaValle, S.M., 2011, September. Space-filling trees: A new perspective on incremental search for motion planning. In 2011 IEEE/RSJ International Conference on Intelligent Robots and Systems (pp. 2199-2206). IEEE.

 \bibitem{RRT}
LaValle, S.M. and Kuffner Jr, J.J., 2001. Randomized kinodynamic planning. The international journal of robotics research, 20(5), pp.378-400.

  \bibitem{galceran}
Galceran, E. and Carreras, M., 2013. A survey on coverage path planning for robotics. Robotics and Autonomous systems, 61(12), pp.1258-1276.
\bibitem{choset}
Choset, H., 2001. Coverage for robotics–a survey of recent results. Annals of mathematics and artificial intelligence, 31(1), pp.113-126.
\bibitem{uavsurvey}
Cabreira, T.M., Brisolara, L.B. and Paulo R, F.J., 2019. Survey on coverage path planning with unmanned aerial vehicles. Drones, 3(1), p.4.


 \bibitem{spires}
Spires, S.V. and Goldsmith, S.Y., 1998, July. Exhaustive geographic search with mobile robots along space-filling curves. In International Workshop on Collective Robotics (pp. 1-12). Springer, Berlin, Heidelberg.



 \bibitem{sadat2}
 Sadat, S.A., Wawerla, J. and Vaughan, R., 2015, May. Fractal trajectories for online non-uniform aerial coverage. In 2015 IEEE international conference on robotics and automation (ICRA) (pp. 2971-2976). IEEE.

 \bibitem{sadat1}
Sadat, S.A., Wawerla, J. and Vaughan, R.T., 2014, September. Recursive non-uniform coverage of unknown terrains for uavs. In 2014 IEEE/RSJ International Conference on Intelligent Robots and Systems (pp. 1742-1747). IEEE.


\bibitem{hawk}
Z. Liu, Y. Chen, B. Liu, C. Cao and X. Fu, "HAWK: An Unmanned Mini-Helicopter-Based Aerial Wireless Kit for Localization," in IEEE Transactions on Mobile Computing, vol. 13, no. 2, pp. 287-298, Feb. 2014, doi: 10.1109/TMC.2012.238.
\bibitem{tiwari}
Tiwari, A., Chandra, H., Yadegar, J. and Wang, J., 2007. Constructing optimal cyclic tours for planar exploration and obstacle avoidance: A graph theory approach. In Advances in Cooperative Control and Optimization (pp. 145-165). Springer, Berlin, Heidelberg. 
\bibitem{ban}
Ban, X., Goswami, M., Zeng, W., Gu, X. and Gao, J., 2013, April. Topology dependent space filling curves for sensor networks and applications. In 2013 Proceedings IEEE INFOCOM (pp. 2166-2174). IEEE.
\bibitem{nair}
S. H. Nair, A. Sinha and L. Vachhani, "Hilbert's space-filling curve for regions with holes," 2017 IEEE 56th Annual Conference on Decision and Control (CDC), 2017, pp. 313-319, doi: 10.1109/CDC.2017.8263684.
\bibitem{joshi}
Joshi, A.A., Bhatt, M.C. and Sinha, A., 2019, December. Modification of Hilbert’s Space-Filling Curve to Avoid Obstacles: A Robotic Path-Planning Strategy. In 2019 Sixth Indian Control Conference (ICC) (pp. 338-343). IEEE.
\bibitem{wakode}
Wakode, A. and Sinha, A. (2022) ‘Online Evasive Strategy for Aerial Survey using Sierpinski curve’, IFAC-PapersOnLine, 55(22), pp. 129–134. doi: 10.1016/j.ifacol.2023.03.022.

\bibitem{hilbert}
Altay, G. (2021). HilbertCurve: Python implementation of the Hilbert curve. 
\bibitem{igraph}
Csardi, G., \& Nepusz, T. (2006). The igraph software package for complex network research. InterJournal, Complex Systems, 1695.






\end{thebibliography}
\end{document}